\documentclass{article}

\usepackage[preprint,nonatbib]{neurips_2026} 
\usepackage{cite}                    

\usepackage[utf8]{inputenc} 
\usepackage[T1]{fontenc}    
\usepackage[
  colorlinks=true,
  urlcolor=blue,
  linkcolor=black,
  citecolor=black
]{hyperref}

\usepackage{array}
\usepackage{url}            
\usepackage{booktabs}       
\usepackage{amsfonts}       
\usepackage{nicefrac}       
\usepackage{microtype}      
\usepackage{multirow}       
\usepackage{makecell}       
\usepackage{amsmath,amssymb,amsthm,graphicx}
\usepackage[ruled,vlined,linesnumbered]{algorithm2e}
\usepackage{algpseudocode}   
\usepackage{wrapfig} 
\usepackage{tikz}
\usepackage[table]{xcolor}

\definecolor{rowblue}{RGB}{230,242,255}

\newtheorem{lemma}{Lemma}
\newtheorem{theorem}{Theorem}

\title{IB-Flow: Information Bottleneck-Guided CFG Distillation for Few-Step Text-to-Image Generation
\thanks{\url{https://github.com/zhjy2016/IBFlow}}
}

\author{%
  Yiting Wang\textsuperscript{1} \quad 
  Jingyi Zhang\textsuperscript{2} \quad 
  Wenhu Zhang\textsuperscript{3} \quad 
  Ke Chao\textsuperscript{4} \\ 
  Yves Liang\textsuperscript{1
  } \quad 
  Kun Cheng\textsuperscript{2} \quad 
  Kang Zhao\textsuperscript{2}\thanks{Corresponding author.} \\
  \\ 
  \textsuperscript{1}Tsinghua University \quad \textsuperscript{2}Wan Team, Alibaba Group  \quad
  \textsuperscript{3}HKUST \quad \textsuperscript{4}Beijing Normal University 
}

\begin{document}

\maketitle

\begin{abstract}
While large-scale text-to-image generative models have achieved unprecedented visual performance, their inherent reliance on multi-step iterative solvers incurs severe inference latency. Few-step distillation targeting the Classifier-Free Guidance (CFG) trajectory has emerged as the prevalent dual-dimensional compression paradigm. However, existing frameworks remain subjugated by a coarse-grained blind injection paradigm that perpetually enforces a globally static guidance strength while indiscriminately sampling the supervisor timestep. This state-agnostic design completely disregards the intrinsic nature of image generation as a dynamic evolutionary process characterized by progressive entropy reduction, which not only restricts the performance boundary of few-step compression but also precipitates severe CFG over-conditioning artifacts. To transcend these limitations, we re-examine the distillation procedure through the theoretical lens of Information Theory, formally modeling it as a dynamic mutual information game constrained by the Information Bottleneck (IB) principle. Specifically, we dismantle traditional blind assumptions via a dual-track adaptive framework. To determine the injection target, we propose an instance-aware selection mechanism that transmutes the intractable KL divergence constraint into a zero-overhead closed-form solution predicated on the local vector field norm. To regulate the injection strength, we introduce an entropy-aware schedule that dynamically decays alongside the Signal-to-Noise Ratio (SNR), applying maximal thrust for initial structural anchoring before smoothly reverting to the natural manifold to refine micro-details. Extensive empirical evaluations corroborate that our framework fundamentally eradicates over-conditioning artifacts, shattering the performance ceiling to achieve state-of-the-art (SOTA) generative fidelity under extremely stringent 2-step configurations.


\end{abstract}  
\begin{figure*}[!t]
    \centering
    \includegraphics[width=1\textwidth]{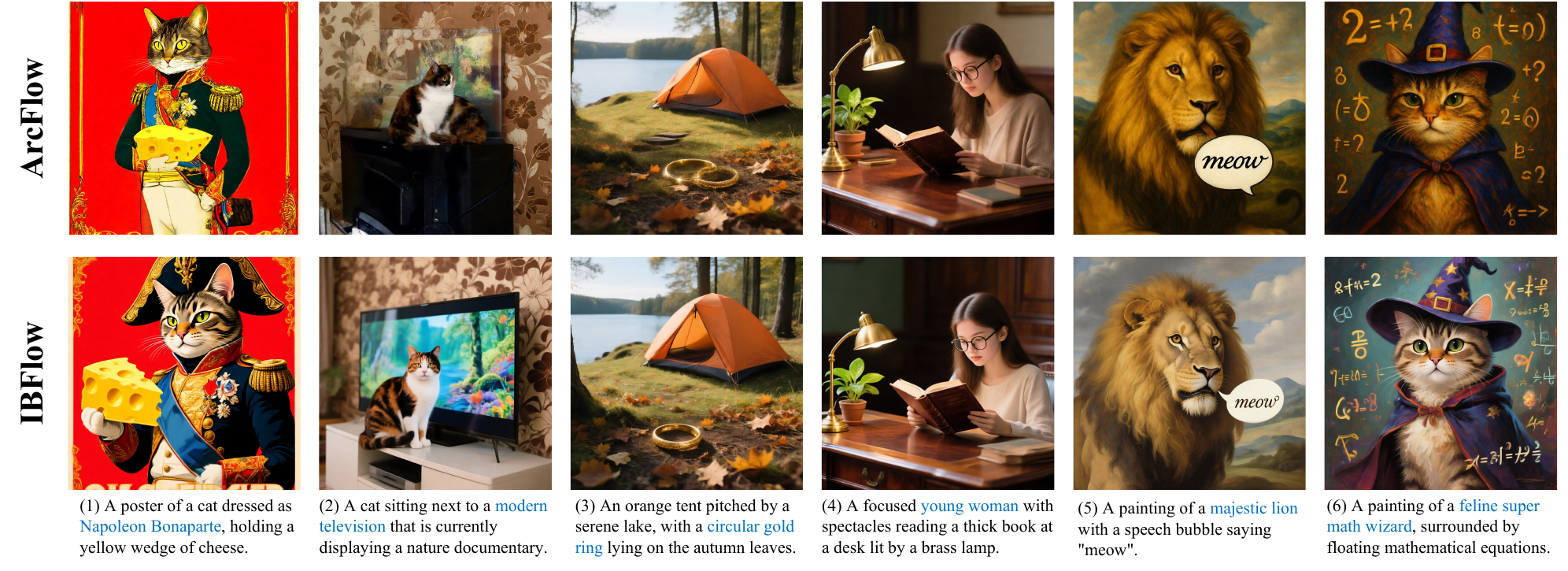} 
    
    \caption{Demonstrating extreme few-step image generation at only 2 NFE, our IBFlow seamlessly improves perceptually orthogonal qualities compared to Arcflow, which improves from: (1,2) structural fidelity, (3,4) textural details, and (5,6) chromatic fidelity.}
    \label{fig:teaser}
    \vspace{-0.7cm}
\end{figure*}

\section{Introduction}
\label{sec:intro}

Recent advances in large-scale generative models, particularly those based on Diffusion Models \cite{ho2020denoising, rombach2022high} and Flow Matching \cite{lipman2022flow}, have achieved unprecedented success in text-to-image generation. By constructing continuous probability paths from simple prior distributions to complex data distributions, they enable highly realistic visual synthesis. However, these models heavily rely on iterative numerical solvers during the sampling phase, typically requiring dozens or even hundreds of Neural Function Evaluations (NFEs). This sequential denoising nature leads to severe inference latency and massive computational overhead, fundamentally limiting their deployment in real-time interactive applications.

To break the inference efficiency bottleneck of diffusion models, the core objective of few-step generation fundamentally necessitates a dual-dimensional compression paradigm, wherein step distillation truncates the multi-step trajectory into a mere one or two steps \cite{yin2024dmd, salimans2022progressive, song2023consistency}, while CFG distillation simultaneously compresses the conditional axis to eliminate the doubled forward-pass overhead imposed by Classifier-Free Guidance (CFG) \cite{ho2022classifier}. Recent pioneering studies such as Decoupled DMD \cite{yin2024dmd2, liu2025decoupled} unveil a profound entanglement between these two dimensions by demonstrating that CFG distillation transcends mere computational reduction to act as the quintessential spear propelling the model toward large-stride generation and significantly elevating the quality ceiling of few-step synthesis. Consequently, the paramount challenge transitions into losslessly condensing the teacher model's burdensome CFG-dependent guidance capability into the student model through a rigorous feature injection constraint co-determined by the \textbf{supervisor timestep} and the \textbf{guidance strength}.

Despite significant progress in compressing guidance trajectories, existing frameworks rely on a blind injection paradigm that perpetually enforces a globally static guidance strength while indiscriminately sampling the supervisor timestep from empirical priors. This state-agnostic design disregards the intrinsic nature of image generation as a dynamic, progressive entropy-reduction process. In the high-entropy initial stages, the chaotic representation urgently demands conservative structural anchoring; blindly leaping toward distant targets creates a massive semantic gap that precipitates severe noise degradation \cite{liu2025decoupled}. Conversely, during the low-entropy terminal stages when the focus shifts to micro-detail refinement, maintaining rigid guidance disrupts the converged natural manifold, manifesting severe CFG over-conditioning artifacts like color oversaturation and texture sharpening. Ultimately, this rigid adherence to static injection fails to unlock the CFG engine's potential, emerging as the primary impediment to the performance boundary of few-step generative models.

To resolve these limitations, we re-examine the few-step distillation procedure through the theoretical lens of Information Theory to propose a dynamic information injection framework governed by the Information Bottleneck (IB) principle \cite{tishby1999information}. Specifically, we dismantle the traditional blind assumptions from two dimensions: \textbf{injection target} and \textbf{injection strength}. On the one hand, to answer \textit{what target to inject}, we reformulate the selection of the supervisor timestep as a capacity-constrained mutual information maximization problem that transmutes the intractable KL divergence constraint into a zero-overhead closed-form solution predicated on the local vector field norm. This instance-aware mechanism dynamically enforces conservative structural anchoring early on before facilitating advanced detail absorption tailored to specific sample complexities. On the other hand, to answer \textit{how much strength to inject}, we adhere to the minimal sufficient representation tenet of IB to introduce an entropy-aware guidance strength schedule. By dynamically decaying alongside the Signal-to-Noise Ratio (SNR), it applies maximal thrust for initial symmetry breaking before smoothly reverting to the unconditional natural manifold, thereby eradicating over-conditioning artifacts. Extensive empirical evaluations corroborate that our framework achieves superlative generative fidelity under extremely stringent low-step configurations by completely shattering the performance ceiling inherent in traditional static distillation paradigms. 

The main contributions of this work are summarized as follows:

(1) We expose the fundamental limitations of blind information injection in traditional CFG distillation from first principles, formally modeling the few-step distillation as a dynamic mutual information game constrained by the Information Bottleneck for the first time.

(2) We propose an \textit{instance-aware} dynamic target selection mechanism. Through rigorous mathematical derivation, we prove that the conditional vector field norm can serve as an exact proxy for semantic uncertainty, yielding a zero-overhead closed-form optimal solution.

(3) We introduce an \textit{entropy-aware} dynamic guidance strength scheduling strategy, which perfectly balances the alignment of conditional generation with the fidelity of the natural image manifold, effectively eradicating CFG over-conditioning artifacts.

(4) Extensive quantitative and qualitative experiments demonstrate that our method achieves State-of-the-Art (SOTA) generation quality and text-alignment under the extreme 2-step generation setting.

\section{Related Work}
\label{sec:related_work}

\subsection{Diffusion and Flow Matching Models}
\label{subsec:rw_diffusion_flow}
Diffusion models achieve state-of-the-art generative performance by learning to reverse an iterative noise-injection process \cite{sohl2015deep, ho2020denoising, song2020score}. Building on this, Flow Matching (FM) \cite{lipman2022flow} and Rectified Flow \cite{liu2022flow} provide a unified framework based on continuous normalizing flows. By directly regressing the vector field from noise to data, FM yields straighter, optimal transport-guided trajectories that simplify training and elevate large-scale text-to-image synthesis. However, both SDE-based diffusion and ODE-based FM require dozens to hundreds of neural function evaluations (NFEs) during sampling. This heavy computational burden fundamentally hinders their real-time deployment \cite{liu2025decoupled}.

\subsection{Few-Step Diffusion Distillation}
\label{subsec:rw_few_step}
To address this latency, extensive research focuses on compressing multi-step models into few-step generators \cite{salimans2022progressive, gu2025starflow, geng2025mean}. Key routes include consistency distillation \cite{song2023consistency, wang2024phased, zheng2025large}, enforcing trajectory self-consistency, and adversarial distillation \cite{sauer2024add, lin2024sdxl}, using GAN objectives to match real data distributions. Among these, score-based methods like Diff-Instruct \cite{luo2023diff} and Distribution Matching Distillation (DMD) \cite{yin2024dmd, yin2024dmd2} are particularly promising \cite{liu2025decoupled}. DMD minimizes an Integral Kullback-Leibler (IKL) divergence between student and teacher distributions, providing an elegant theoretical framework. Building on this, recent DMD variants and trajectory-matching approaches have successfully achieved high-resolution few-step generation \cite{liu2023instaflow, zhu2024slimflow, kim2024simple}. 

\subsection{Classifier-Free Guidance and its Distillation}
\label{subsec:rw_cfg}
Classifier-Free Guidance (CFG) \cite{ho2022classifier, saharia2022photorealistic, ho2022video} is pivotal for achieving high-fidelity text-to-image alignment but doubles the per-step computational cost. Integrating CFG into distillation is therefore crucial, as teacher CFG patterns act as the primary driver for few-step conversion \cite{liu2025decoupled, liu2026self, kim2024simple, luhman2021knowledge}. Notably, Decoupled DMD \cite{liu2025decoupled} rigorously isolates this into a Distribution Matching (DM) regularizer and a CFG Augmentation (CA) engine. However, existing implementations universally adopt a globally static CFG scale and randomly sample the CA supervisor timestep. This blind injection paradigm ignores the dynamic entropy-reduction during image generation, inevitably causing severe over-conditioning artifacts \cite{ho2022classifier, liu2025decoupled, sauer2024add, lin2024sdxl}. 


\begin{figure*}[!t]
    \centering
    \includegraphics[width=0.95\textwidth]{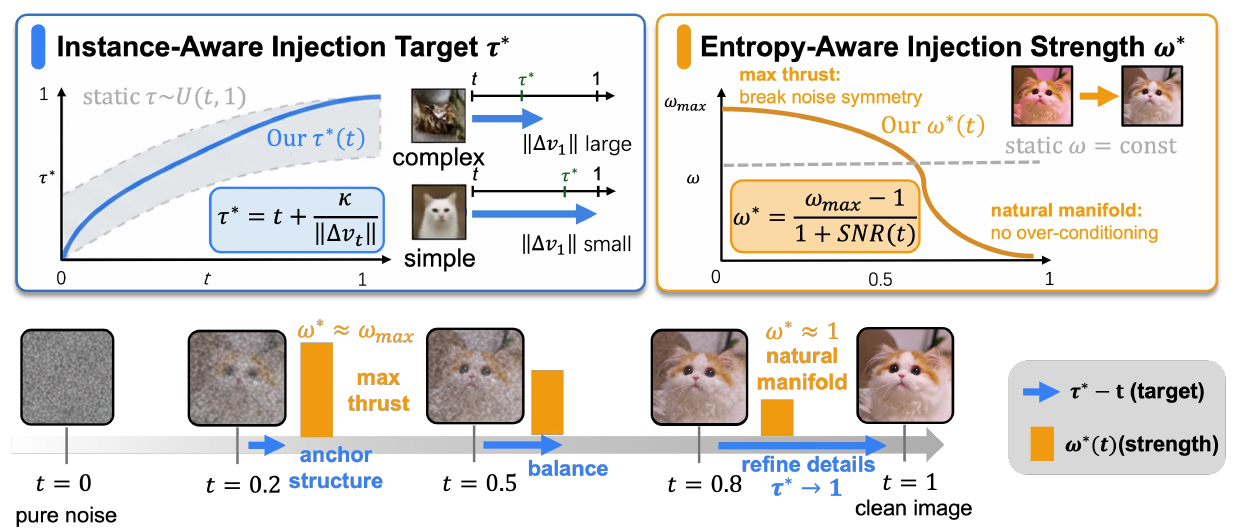}
    \caption{Overview of our dynamic information injection framework. While traditional distillation enforces a static guidance strength $\omega$ and blindly samples the supervisor timestep $\tau$, our IB-guided approach dynamically selects the optimal injection target $\tau^*$ and strength $\omega^*$.}
    \label{fig:method_pipeline}
    \vspace{-0.5cm}
\end{figure*}

\section{Method}
\label{sec:methodology}

In this section, we formally delineate our dynamic information injection framework for ultra-efficient few-step generation. We first establish the theoretical bedrock by revisiting Flow Matching and CFG distillation. To accelerate iterative denoising, few-step distillation has been widely used in an attempt to compress tedious multi-step trajectories (e.g., 50 steps) into very low-step trajectories (e.g., 2 steps) without compromising the fidelity of generation.

\subsection{Preliminaries}
\label{subsec:preliminaries}

\subsubsection{Flow Matching}
\label{subsubsec:flow_matching}

Flow Matching (FM)\cite{lipman2022flow} provides a unified framework based on continuous normalizing flows to learn the probability density path from a simple noise distribution $p_0$ to a complex data distribution $p_1$. Compared to traditional diffusion models, FM simplifies the training objective by directly regressing the vector field, allowing the sampling process to be performed by solving an Ordinary Differential Equation (ODE). Specifically, given data $x_1 \sim p_1$ and noise $x_0 \sim p_0$, we define a linear probability path $x_t = (1-t)x_0 + tx_1$. The conditional vector field corresponding to this path is defined as:
\begin{equation}
    u_t(x_t|x_1) = \frac{\mathrm{d}x_t}{\mathrm{d}t} = x_1 - x_0.
\end{equation}
The training objective for the parameterized neural network $v_\theta(x_t, t)$ is to minimize the expected $L_2$ error between its prediction and the target vector field:
\begin{equation}
    \mathcal{L}_{FM}(\theta) = \mathbb{E}_{t \sim \mathcal{U}(0,1), x_0, x_1} \left\| v_\theta(x_t, t) - (x_1 - x_0) \right\|^2.
\end{equation}
During inference, generation quality heavily relies on the number of ODE solver steps. Although FM offers smoother trajectories than diffusion models, generating high-quality images with large-scale models typically still requires 40-100 Number of Function Evaluations (NFEs), limiting its application in real-time interactive scenarios.

\subsubsection{Few-Step CFG Distillation and Decoupled Supervision}
\label{subsubsec:few_step_cfg_distillation}

CFG is a crucial technique for enhancing text-to-image generation alignment. It strengthens semantic guidance by linearly extrapolating between the conditional vector field $v_t^c(x_t)$ and the unconditional vector field $v_t^u(x_t)$:
\begin{equation}
    \tilde{v}_t^c(x_t, \omega) = v_t^u(x_t) + \omega \left( v_t^c(x_t) - v_t^u(x_t) \right),
    \label{eq:cfg}
\end{equation}
where $\omega > 1$ is the guidance scale. Although CFG significantly improves image quality, it doubles the already high computational cost by requiring two forward passes at each inference step.

Few-step distillation aims to compress the costly multi-step trajectory into extremely few steps. Traditional frameworks \cite{yin2024dmd2} typically attempt to make the student model implicitly fit the modified trajectory. However, recent insights from Decoupled DMD \cite{liu2025decoupled} resolves this issue by rigorously decoupled into Distribution Matching (DM) and CFG Augmentation (CA). The gradient direction driving the student model parameter $\theta$ updates can be approximated as:
\begin{equation}
    \Delta \theta \propto \mathbb{E}_{t, x_t, c} \left[ \Big( \underbrace{\Delta_{DM}(\tau_{DM})}_{\text{Shield (Regularizer)}} + \underbrace{\Delta_{CA}(\tau_{CA}, \omega)}_{\text{Spear (Engine)}} \Big) \frac{\partial x_t}{\partial \theta} \right],
    \label{eq:decoupled_dmd}
\end{equation}
where the DM term $\Delta_{DM}(\tau_{DM}) = v_{\tau_{DM}}^{real, c}(x_{\tau_{DM}}) - v_{\tau_{DM}}^{fake, c}(x_{\tau_{DM}})$ acts as a regularizing shield to suppress generative artifacts by matching student and teacher distributions at the reference timestep $\tau_{DM}$. More crucially, the CA functions as the core spear for achieving few-step compression:
\begin{equation}
    \Delta_{CA}(\tau_{CA}, \omega) = (\omega - 1) \left( v_{\tau_{CA}}^{real, c}(x_{\tau_{CA}}) - v_{\tau_{CA}}^{real, u}(x_{\tau_{CA}}) \right).
    \label{eq:ca_term}
\end{equation}
The CA term distill the deterministic CFG decision pattern of the teacher model at a specific moment into the student model, where the guidance strength $\omega$ governing the injection magnitude and the supervisor timestep $\tau_{CA}$ determining the injection target are two control knobs of CA.

\subsection{Instance-Aware Injection Target via Information Bottleneck}
\label{subsec:proposed_method}

However, existing frameworks remain subjugated by a blind injection paradigm that perpetually enforces a constant guidance strength $\omega$ while randomly sampling the supervisor timestep $\tau_{CA}$ from a broad distribution. As image generation is a dynamic entropy reduction process, where the demand for both \textbf{injection strength} $\omega$ and \textbf{injection target} $\tau_{CA}$ inherently varies across different denoising stages $t$. This static CA injection severely limits the performance upper bound of few-step generation.

\subsubsection{The Information Bottleneck Formulation}

In a single forward pass of CFG distillation, the student at time step $t$ absorbs guidance from the teacher by sampling the teacher at time step $\tau_{CA} > t$. A tradeoff is in this process: on the one hand, teachers should provide enough \textit{information supply} for students to get enough guidance. On the other hand, the distribution difference between students and teachers must not be too large, otherwise it will bring high \textit{learning costs}.

\textbf{Quantifying the Information Supply:} 
The ultimate goal of information injection is to aid the current state $X_t$ in approximating the clean data state $X_1$. Thus, the effective information of the injected features from $\tau_{CA}$ can be rigorously quantified by the conditional mutual information:
\begin{equation}
    \mathcal{I}_{supply}^{\tau_{CA}}(X_1, X_t) = I(X_{\tau_{CA}} ; X_1 | X_t).
\end{equation}

\textbf{Bounding the Learning Cost:} 
We describe the learning cost for students as the difference in distribution between the student and the teacher. In particular, this cost is quantified by the KL divergence between the two marginal distributions \cite{cover2006elements}:
\begin{equation}
    D_{KL}(p_t \| p_{\tau_{CA}}) = \int p_t(x) \log \frac{p_t(x)}{p_{\tau_{CA}}(x)} dx.
    \label{eq:kl_divergence}
\end{equation}
where $p_t(x)$ and $p_{\tau_{CA}}(x)$ denote the marginal probability density functions of the generative process at the current evolutionary state $t$ and the target injection stage $\tau_{CA}$, respectively. 

\textbf{The Information Bottleneck in Distillation:} Due to the monotonic diffusion process of flow matching, the predictive information regarding $X_1$ contained in $X_{\tau_{CA}}$ strictly increases with $\tau_{CA}$. If we merely sought to maximize the supply, we should constantly select $\tau_{CA}=1$ to acquire the most distinct high-frequency supervision. However, directly learning $X_1$ at the early stage brings extremely high learning cost, which can lead to distillation divergence \cite{liu2025decoupled}. To solve this problem, we introduce the \textit{Information Bottleneck (IB)} ~\cite{tishby1999information} and formalize the optimal target time step $\tau_{CA}^*$selection problem as a mutual information maximization problem with cost constraints:
\begin{align}
    \tau_{CA}^*(x_t, t) = \arg\max_{\tau \in [t, 1]} \quad & I(x_{\tau_{CA}} ; X_1 | X_t) \nonumber \\
    \text{s.t.} \quad & D_{KL}(p_t \| p_{\tau_{CA}}) \le \delta.
    \label{eq:ib_optimization}
\end{align}
Where hyperparameter $\delta$ acts as a restricted instantaneous channel capacity. This formulation compactly reveals the mathematical essence of dynamic injection: during the high-entropy early stages, the divergence constraint forces a lower upper bound on $\tau_{CA}$, ensuring the model absorbs safe, low-frequency structural features. Conversely, in the low-entropy later stages, the constraint relaxes, allowing the model to transition towards $\tau \to 1$ to fully leverage mutual information.

\subsubsection{Tractable Approximation via Fisher Information}

Based on \eqref{eq:ib_optimization}, we need to obtain the analytic form of $\tau_{CA}^*(x_t, t)$. However, directly resolving the continuous KL divergence constraint formulated in  within high-dimensional pixel space is profoundly intractable. To circumvent this bottleneck, we adopt the continuous vector field perspective inherent in FM and leverage Taylor expansion to transmute constraint into a zero-overhead closed-form solution. Specifically, we have the following lemma:

\begin{lemma}\label{lem:KL divergence constraint}
The KL divergence constraint in \eqref{eq:ib_optimization} can be approximate as:
\begin{equation}
    D_{KL}(p_t \| p_\tau) \approx \frac{1}{2} (\tau - t)^2 \| \Delta v_t(x_t) \|_2^2.
    \label{eq:kl_fisher_approximation}
\end{equation}
\end{lemma}
\begin{proof}
For an infinitesimal temporal increment $\Delta t = \tau - t$, the local KL divergence can be approximated via Taylor expansion utilizing the Fisher Information Metric (FIM). Under the CFG mechanism, the residual between the conditional and unconditional vector fields, denoted as $\Delta v_t(x_t) = v_t^c(x_t) - v_t^u(x_t)$, is mathematically rigorously proportional to the logarithmic gradient of Pointwise Mutual Information (PMI). Consequently, the divergence penalty precipitated by information injection can be strictly bounded by the squared $L_2$ norm of the local vector field. The full proof can be found in Appendix \ref{sec:appendix_lemma1}.
\end{proof}

Further, we consider the allowable range of $D_{KL}(p_t \| p_\tau)$, which, by Lemma \ref{lem:KL divergence constraint}, depends on the allowable upper bound of the injection span at the target time step $\tau$. Specifically, we have the following lemma:

\begin{lemma}\label{lem:injection span}
The injection span at the target time step $\tau$ satisfies:
\begin{equation}
    \tau - t \le \frac{\sqrt{2\delta}}{\| \Delta v_t(x_t) \|_2}.
\end{equation}
\end{lemma}

Finally, to maximize the mutual information supply, we adopt a greedy strategy by taking the equality. Combining Lemma 1 and Lemma 2, we have the following theorem:
\begin{theorem}\label{the:tau}
The closed-form optimal solution of $\tau_{CA}^*$ is:
\begin{equation}
    \tau_{CA}^*(x_t, t) = \min \left( 1, \quad t + \frac{\kappa}{\| v_t^c(x_t) - v_t^u(x_t) \|_2 + \epsilon} \right),
    \label{eq:final_tau_optimal}
\end{equation}
where $\kappa = \sqrt{2\delta}$ emerges as a hyperparameter encapsulating the globally divergence budget to explicitly govern the adaptive injection stride, while $\epsilon$ serves as a numerical stabilizing constant.
\end{theorem}


The proof can be found in Appendix \ref{sec:appendix_lemma2} and \ref{sec:appendix_theorem1}. Grounded in this formulation, we achieve an Information Bottleneck-guided dynamic optimization for CFG injection target $\tau$ across various generative stages $t$ where enforcing conservative macroscopic structural anchoring during the high-entropy initial phases, while facilitating meticulous micro-detail refinement in the solidified latter stages. By supplying adaptive and stabilized supervision for few-step CFG distillation, our framework fundamentally shatters the performance ceiling and elevates the generative fidelity of synthesized images.

\subsection{Entropy-Aware Dynamic Injection Strength}
\label{subsec:dynamic_omega}

While the instance-aware target $\tau_{CA}^*$ ensures the \emph{injection target} resides within a safe semantic divergence bound, it does not dictate the \emph{injection strength}. Existing distillation work generally employs a static CFG scale $\omega$ across all time steps, and the consistent guidance strength destroys the local characteristics of the aggregation, leading to CFG overconditioning effects, including color supersaturation and abnormal texture sharpening. This highlights that the system's demand for conditional information $C$ varies dynamically.

We consider the IB to quantify the dynamic demand for conditional information $C$. Analogous to the optimization of the injection target, we establish a rigorous theoretical equilibrium between the conditioning cost $I(X_t^{\omega} ; C)$ and the resultant generative fidelity $I(X_t^{\omega} ; X_1 | C)$:
\begin{equation}
    \min_{\omega(t)} \quad \underbrace{I(X_t^{\omega} ; C)}_{\text{Conditioning Cost}} - \beta \underbrace{I(X_t^{\omega} ; X_1 | C)}_{\text{Generative Fidelity}},
    \label{eq:ib_condition_optimization}
\end{equation}
where $X_t^{\omega}$ denotes the updated internal state driven by the guidance strength $\omega$, and $\beta$ is a hyperparameter balancing the two terms.

Based on \eqref{eq:ib_optimization}, we obtain  the analytic form of $\omega^*(t)$ as follows:

\begin{theorem}\label{the:omega}
The closed-form optimal solution of $\omega^*(t)$ is:
\begin{equation}
    \omega^*(t) = 1 + (\omega_{max} - 1) \cdot \frac{1}{1 + \gamma \cdot \text{SNR}(t)}, \quad \text{where} \quad \text{SNR}(t) = \frac{t^2}{(1 - t)^2},
    \label{eq:final_dynamic_omega}
\end{equation}
where $\omega_{max}$ designates the maximum guidance scale required to shatter the structural symmetry of pure noise in the initial high-entropy stage, and $\gamma$ is a temperature hyperparameter meticulously calibrating the sensitivity to the SNR transition.
\end{theorem}
\begin{proof}
At any timestep $t$, the effective information gain achievable by increasing $\omega$ is bounded by the system's current residual uncertainty regarding the real data $X_1$, formally expressed as the conditional entropy $H(X_1 | X_t)$. Based on the fundamental I-MMSE relation \cite{guo2005mutual}, this residual uncertainty is analytically proven to be inversely proportional to the system's instantaneous Signal-to-Noise Ratio (SNR). Consequently, the degree of unknownness in the generative process can be explicitly quantified by the precise evolution of the SNR. Furthermore, to maximize generative fidelity while preventing manifold distortion caused by over-fitting the condition $C$, the supplementary driving force injected by the CFG engine, defined as $\omega(t) - 1$, must be proportional to the current residual uncertainty. By coupling this proportionality with the I-MMSE theorem, we substitute residual uncertainty with the instantaneous SNR. The full proof can be found in Appendix \ref{sec:appendix_theorem2}.
\end{proof}

Theorem \ref{the:omega} mathematically guarantees ideal boundary conditions: as $t \to 1$ and residual uncertainty dissipates, the guidance thrust smoothly and inevitably reverts to the unconditional natural manifold $\omega^*(1) \to 1$, thereby eradicating CFG-induced artifacts from first principles.

\subsection{Dynamic CFG Augmentation Objective}
\label{subsec:overall_loss}

Finally, by substituting our conclusion in Theorem \ref{the:tau} and Theorem \ref{the:omega} into \eqref{eq:ca_term}, we formulate the ultimate distillation objective. The gradient update for the student parameter $\theta$, driven explicitly by our dynamic CA formulation, is rigorously defined as:
\begin{equation}
\Delta \theta_{CA} \propto \mathbb{E}_{t, x_t, c}\!\left[\, \underbrace{(\omega^*(t) - 1)}_{\text{Dynamic Strength}} \underbrace{\left( v_{\tau_{CA}^*}^{real,c}(x_{\tau_{CA}^*}) - v_{\tau_{CA}^*}^{real,u}(x_{\tau_{CA}^*}) \right)}_{\text{Dynamic Target}} \nabla_\theta v_\theta(x_t, c) \,\right].
\end{equation}
This theoretically grounded, parameter-free objective elegantly unifies macroscopic structural anchoring and microscopic detail refinement under a strict entropy-reduction trajectory, establishing a foundation for extreme few-step generative synthesis.

\section{Experiments}
\label{sec:experiments}

\subsection{Experimental setup}
\label{subsec:exp_setup}

We apply the IBFlow distillation framework to three text-to-image teachers spanning different scales and paradigms: FLUX.1-dev~\cite{labs2025flux}, OpenUni-L-512~\cite{wu2025openuni}, and Qwen-Image-20B~\cite{qwenimage}, all of which generate images at $1024{\times}1024$ with classifier-free guidance under their default $50{\times}2$-NFE schedules. Following ArcFlow~\cite{yang2025arcflow}, all students are trained on a unified corpus of 2.3M text prompts. We optimize the model under BF16 mixed precision on 32 A100 GPUs. The CA divergence budget is $\kappa{=}1.5$, the SNR temperature is $\gamma{=}1.0$, $\omega_{\max}=4.0$ is inherited from each teacher's default CFG scale. For evaluation, we follow standard text-to-image protocols and report results on three benchmarks: GenEval~\cite{ghosh2023geneval}, DPG-Bench~\cite{hu2024ella}, and OneIG-Bench~\cite{chang2026oneig} for fine-grained alignment along five axes. We deploy ArcFlow with their official released codebases for fair comparison.




\subsection{Main Results}
\label{subsec:comp_sota}

\begin{table}[t]
  \centering
  \renewcommand\arraystretch{1.3}
  \small
  \caption{Quantitative comparisons on Geneval, DPG-Bench and OneIG-Bench. The NFE of Qwen-Image-20B is recorded as $50 \times 2$ since it uses CFG.}
  \label{tab:main_comparison}
  \resizebox{\textwidth}{!}{
  \begin{tabular}{l c c c c c c c c}
    \toprule
    \multirow{2}{*}{\textbf{Model}} & \multirow{2}{*}{\textbf{NFE}} & \multirow{2}{*}{\textbf{Geneval}} & \multirow{2}{*}{\textbf{DPG-Bench}} & \multicolumn{5}{c}{\textbf{OneIG-Bench}} \\
    \cmidrule(lr){5-9}
    & & & & Alignment & Text & Diversity & Style & Reasoning \\
    \midrule
    FLUX.1-dev & 50 & 0.66 & 84.16 & 0.790 & 0.556 & 0.238 & 0.307 & 0.257 \\
    \midrule
    SenseFlow (FLUX) & 2 & 0.60 & 79.86 & 0.743 & 0.230 & 0.139 & 0.341 & 0.212 \\
    pi-Flow (GM-FLUX) & 2 & 0.58 & 82.36 & 0.764 & 0.141 & 0.216 & 0.332 & 0.212 \\
    ArcFlow (FLUX)  & 2 & 0.65 & 84.29 & 0.798 & 0.368 & 0.210 & 0.350 & 0.224 \\
    \rowcolor{gray!10} \textbf{IBFlow-FLUX (Ours)} & 2 & \textbf{0.68} & \textbf{85.13} & \textbf{0.802} & \textbf{0.465} & \textbf{0.236} & \textbf{0.312} & \textbf{0.247} \\
    \midrule
    \midrule
    OpenUni-L-512 & $20\times2$ & 0.85 & 81.54 & 0.771 & 0.275 & 0.132 & 0.396 & 0.227 \\
    \midrule
    OpenUni-RCGM-512 & 2 & 0.85 & 80.15 & 0.742 & 0.005 & 0.129 & 0.365 & 0.196 \\
    OpenUni-TwinFlow-512 & 2 & 0.85 & 79.82 & 0.735 & 0.006 & 0.126 & 0.372 & 0.219 \\
    \rowcolor{gray!10} \textbf{OpenUni-IBFlow-512 (Ours)} & 2 & \textbf{0.86} & \textbf{81.15} & \textbf{0.756} & \textbf{0.089} & \textbf{0.231} & \textbf{0.385} & \textbf{0.231} \\
    \midrule
    \midrule
    Qwen-Image-20B & $50\times2$ & 0.87 & 88.32 & 0.880 & 0.888 & 0.194 & 0.427 & 0.306 \\
    \midrule
    Qwen-Image-Lightning & 2 & 0.85 & 86.47 & 0.875 & 0.879 & 0.098 & 0.415 & 0.292 \\
    TwinFlow (Qwen) & 2 & 0.82 & 87.01 & 0.862 & 0.825 & 0.130 & 0.364 & 0.267 \\
    pi-Flow (GM-Qwen)~\cite{chen2025piflow} & 2 & 0.83 & 86.45 & 0.837 & 0.634 & 0.176 & 0.382 & 0.259 \\
    ArcFlow (Qwen) & 2 & 0.84 & 87.94 & 0.868 & 0.789 & \textbf{0.182} & 0.418 & 0.283 \\
    \rowcolor{gray!10} \textbf{IBFlow-Qwen (Ours)} & 2 & \textbf{0.86} & \textbf{88.67} & \textbf{0.877} & \textbf{0.863} & 0.180 & \textbf{0.425} & \textbf{0.301} \\
    \bottomrule
  \end{tabular}
  }
\end{table}

\begin{table}[t]
  \centering
  \renewcommand\arraystretch{1.25}
  \small
  \caption{Fine-grained breakdown of DPG-Bench on Qwen-Image-20B at 2 NFE. We report DPG total, the four L1 axes (\emph{global}, \emph{entity}, \emph{relation}, \emph{attribute}), and the most informative L2 axes covering structure, relation, color and detail.}
  \label{tab:dpg_finegrained}
  \resizebox{\textwidth}{!}{
  \begin{tabular}{l c c c c c c c c c c}
    \toprule
    \multirow{2}{*}{\textbf{Method}} & \multirow{2}{*}{\textbf{NFE}} & \multirow{2}{*}{\textbf{DPG total}} & \multicolumn{4}{c}{\textbf{L1 axes}} & \multicolumn{4}{c}{\textbf{L2 axes}} \\
    \cmidrule(lr){4-7}\cmidrule(lr){8-11}
    & & & Global & Entity & Relation & Attribute & Whole & Spatial & Color & Texture \\
    \midrule
    ArcFlow~\cite{yang2025arcflow} & 2 & 87.94 & 83.89 & 93.09 & 94.58 & 90.02 & 94.86 & 94.93 & 94.28 & 90.08 \\
    \rowcolor{gray!10}\textbf{IBFlow-Qwen (Ours)} & 2 & \textbf{88.68} & \textbf{85.11} & \textbf{93.59} & \textbf{95.20} & \textbf{90.76} & \textbf{95.56} & \textbf{95.59} & \textbf{94.93} & \textbf{91.17} \\
    \midrule
    \rowcolor{green!8}\emph{Improvement ($\Delta$)} & -- & \textbf{+0.74} & \textbf{+1.22} & \textbf{+0.50} & \textbf{+0.62} & \textbf{+0.74} & \textbf{+0.70} & \textbf{+0.66} & \textbf{+0.65} & \textbf{+1.09} \\
    \bottomrule
  \end{tabular}
  }
  \vspace{-0.3cm}
\end{table}

\paragraph{Quantitative comparison.}
Table~\ref{tab:main_comparison} compares IBFlow with representative 2-step distillation methods on each backbone, including SenseFlow~\cite{ge2025senseflow}, pi-Flow~\cite{chen2025piflow}, ArcFlow~\cite{yang2025arcflow}, RCGM~\cite{sun2026anystep}, TwinFlow~\cite{cheng2026twinflow}, and Qwen-Image-Lightning~\cite{qwenimagelightning2025}. IBFlow attains the best overall scores across all three backbones at only 2 NFE. For challenging Qwen-Image-20B, IBFlow attains \textbf{0.86} GenEval and \textbf{88.67} DPG-Bench at only 2 NFE, outperforms state-of-the-art ArcFlow on both metrics ($0.84$ / $87.94$), demonstrating that the IB-driven dynamic schedules tighten the injection during high-entropy stages and relax it once the trajectory solidifies, thereby alleviating in one shot the three failure modes that plague every static-CA baseline at 2 NFE. The improvements are more pronounced compared with TwinFlow which have $0.82$ / $87.01$ on GenEval / DPG-Bench, indicating that its self-adversarial recipe still struggles to preserve fine-grained semantic alignment under such an aggressive step budget.

Table~\ref{tab:dpg_finegrained} further dissects the DPG-Bench gain over ArcFlow~\cite{yang2025arcflow} along the official L1/L2 axes. The two largest improvements pinpoint the role of each dynamic schedule: \emph{Global} rises by \textbf{+1.22}, showing that the IB-bounded supervisor $\tau_{CA}^*$ keeps the student anchored to safe low-frequency layouts during the high-entropy early stage; \emph{Texture} rises by \textbf{+1.09}, mirroring the boundary behavior $\omega^*(t)\!\to\!1$ that retracts late-stage guidance and removes over-sharpening. The remaining axes (Entity, Relation, Color) are also improved, confirming that the two complementary IB-driven schedules jointly lift fidelity across the entire entropy-reduction trajectory rather than along a single dimension.

\begin{figure*}[!t]
    \centering
    \begin{minipage}{0.64\textwidth}
        \centering
        \makebox[0.25\textwidth][c]{Lightning}%
        \makebox[0.25\textwidth][c]{ArcFlow}%
        \makebox[0.25\textwidth][c]{TwinFlow}%
        \makebox[0.25\textwidth][c]{\textbf{Ours}}
    \end{minipage}%
    \begin{minipage}{0.02\textwidth}
        \centering
    \end{minipage}%
    \begin{minipage}{0.32\textwidth}
        \centering
        \makebox[0.5\textwidth][c]{w/o Dyn $\tau_{CA}^*$}%
        \makebox[0.5\textwidth][c]{w/o Dyn $\omega^*(t)$}
    \end{minipage}\par\vspace{1mm}

    
    \begin{minipage}{0.64\textwidth}
        \includegraphics[width=\textwidth]{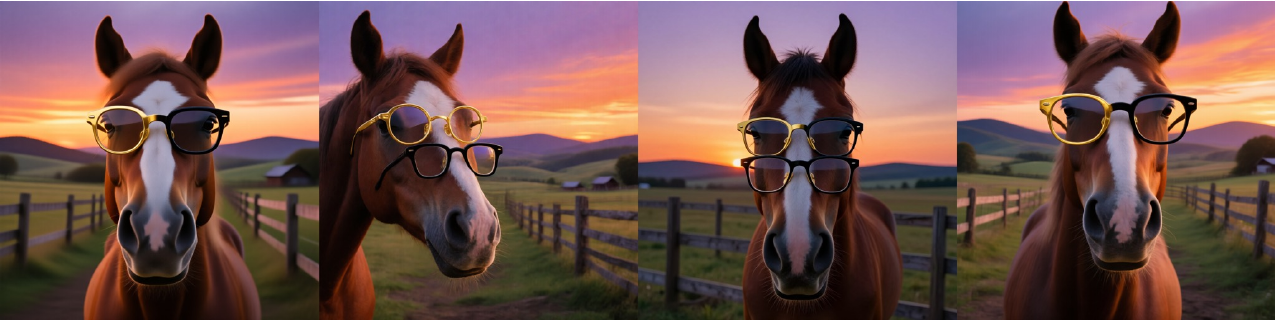}
    \end{minipage}%
    \begin{minipage}{0.02\textwidth}
        \centering
        $\vcenter{\hbox{\tikz[overlay] \draw [dashed, gray, thick] (0, 1.25) -- (0, -1.25);}}$
    \end{minipage}%
    \begin{minipage}{0.32\textwidth}
        \includegraphics[width=0.49\textwidth]{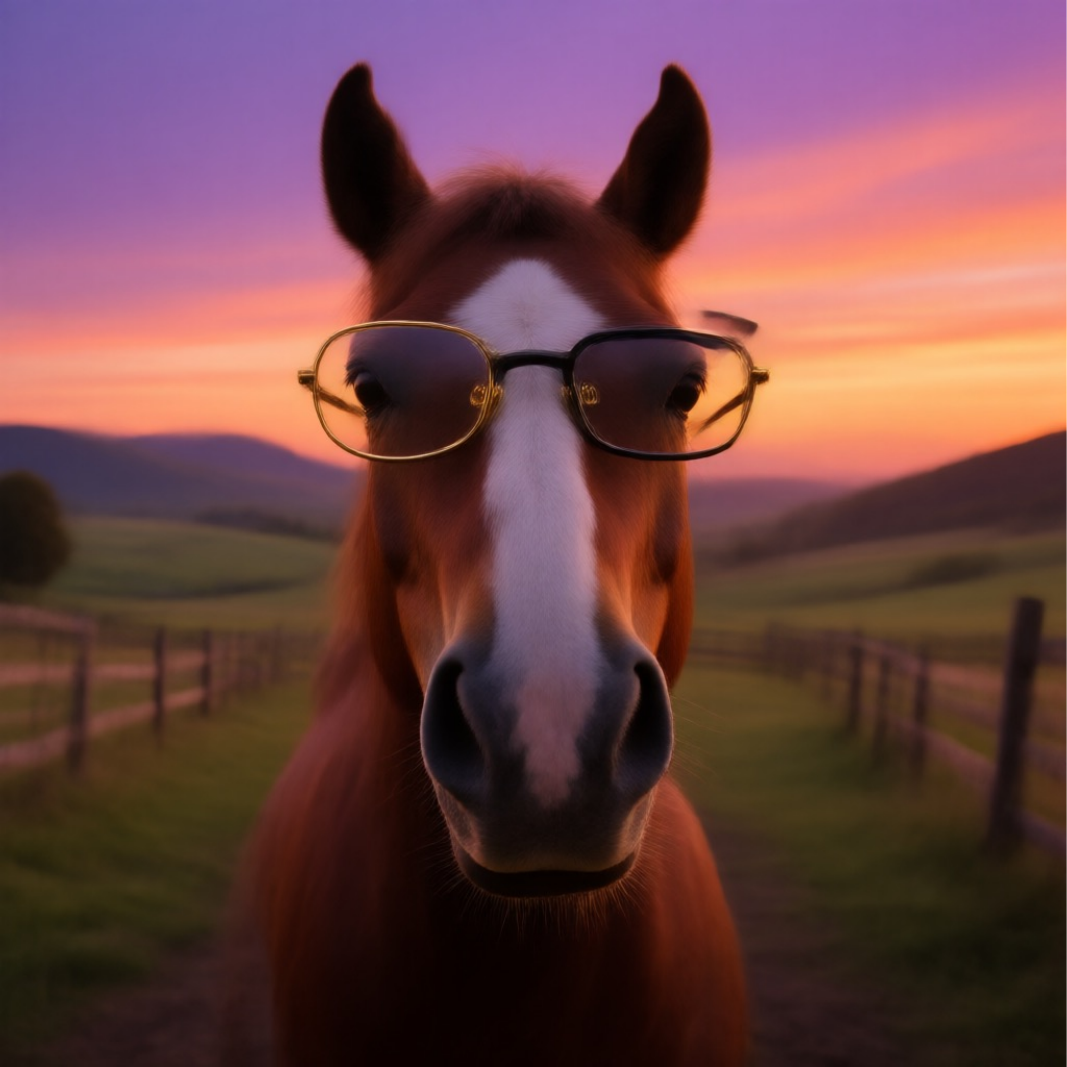}\hfill
        \includegraphics[width=0.49\textwidth]{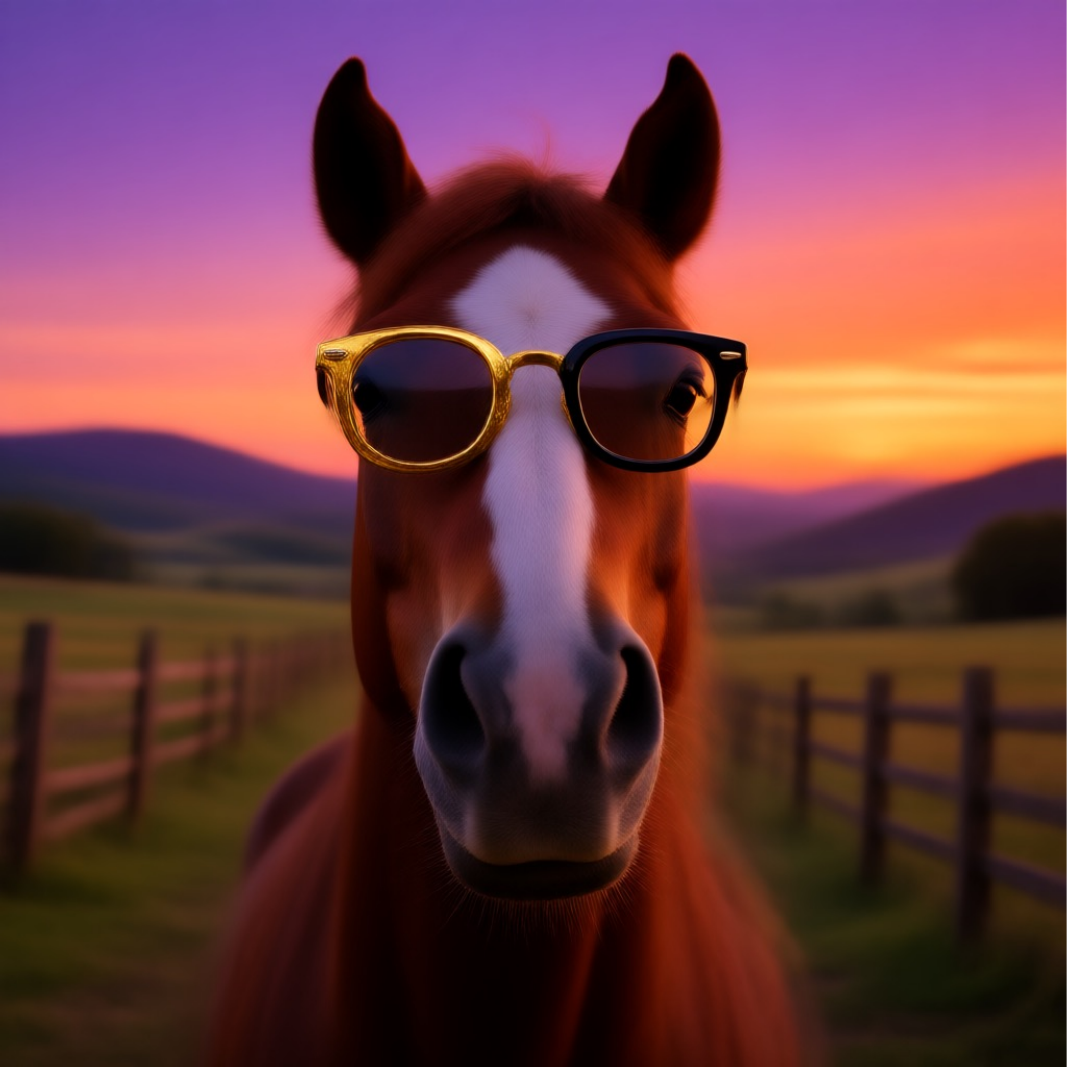}
    \end{minipage}\par\vspace{0.5mm}
    
    \begin{minipage}{0.64\textwidth}
        \includegraphics[width=\textwidth]{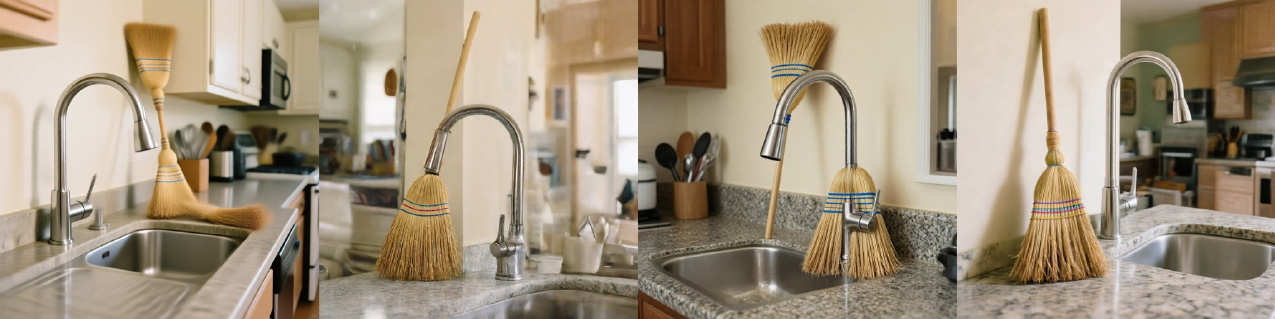}
    \end{minipage}%
    \begin{minipage}{0.02\textwidth}
        \centering
        $\vcenter{\hbox{\tikz[overlay] \draw [dashed, gray, thick] (0, 1.25) -- (0, -1.25);}}$
    \end{minipage}%
    \begin{minipage}{0.32\textwidth}
        \includegraphics[width=0.49\textwidth]{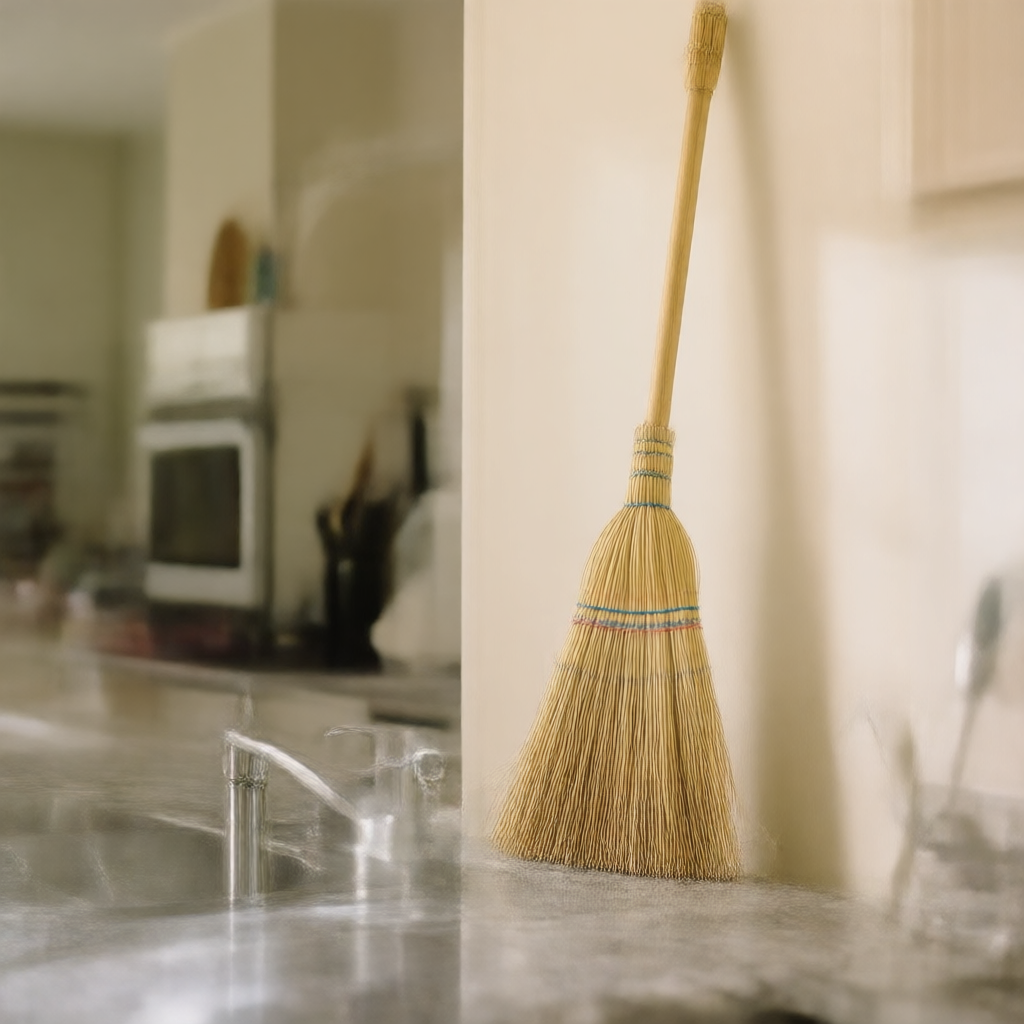}\hfill
        \includegraphics[width=0.49\textwidth]{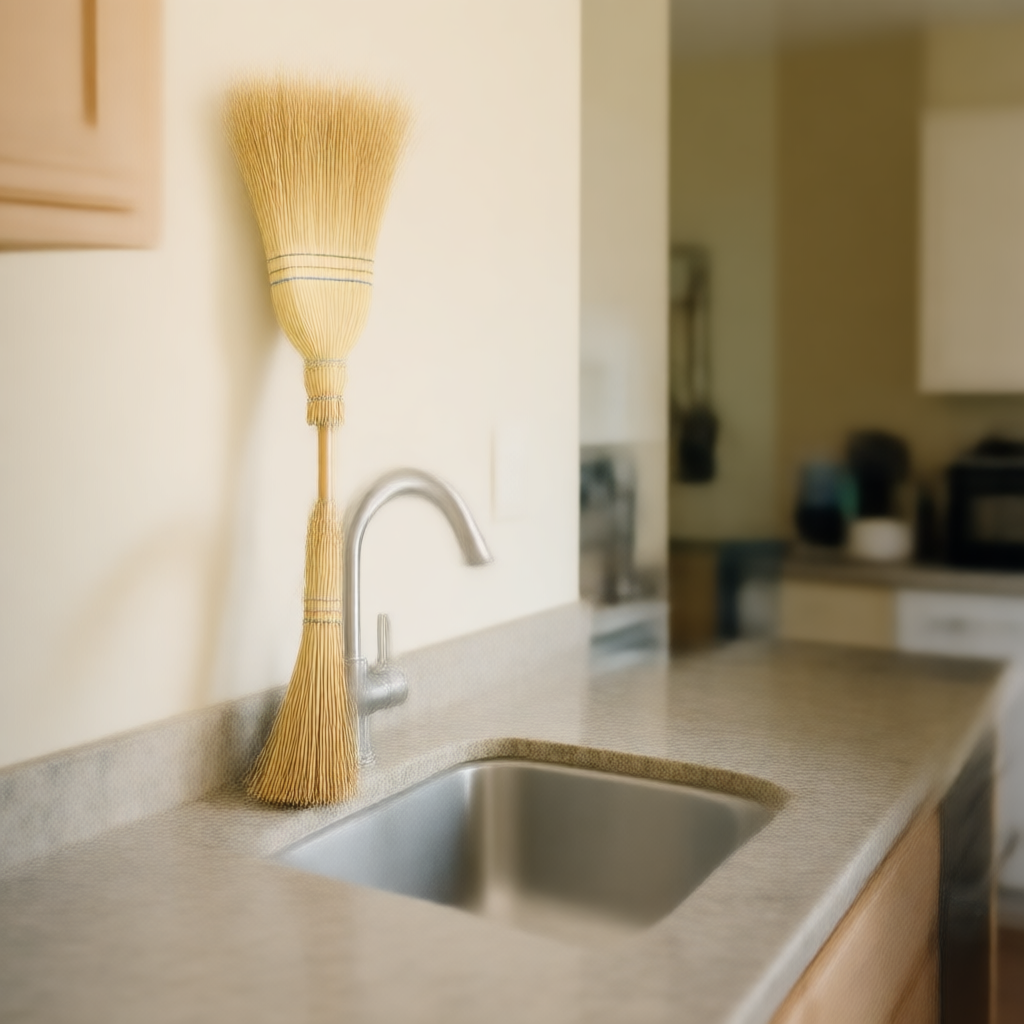}
    \end{minipage}\par\vspace{0.5mm}

    \begin{minipage}{0.64\textwidth}
        \includegraphics[width=\textwidth]{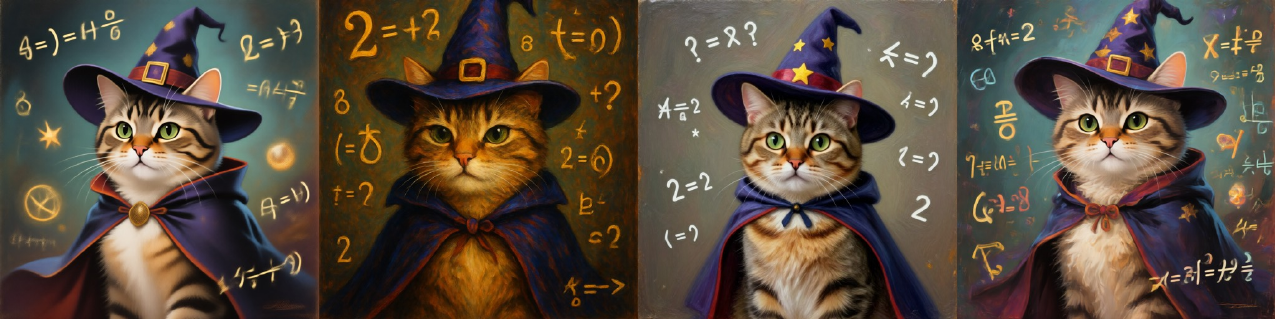}
    \end{minipage}%
    \begin{minipage}{0.02\textwidth}
        \centering
        $\vcenter{\hbox{\tikz[overlay] \draw [dashed, gray, thick] (0, 1.25) -- (0, -1.25);}}$
    \end{minipage}%
    \begin{minipage}{0.32\textwidth}
        \includegraphics[width=0.49\textwidth]{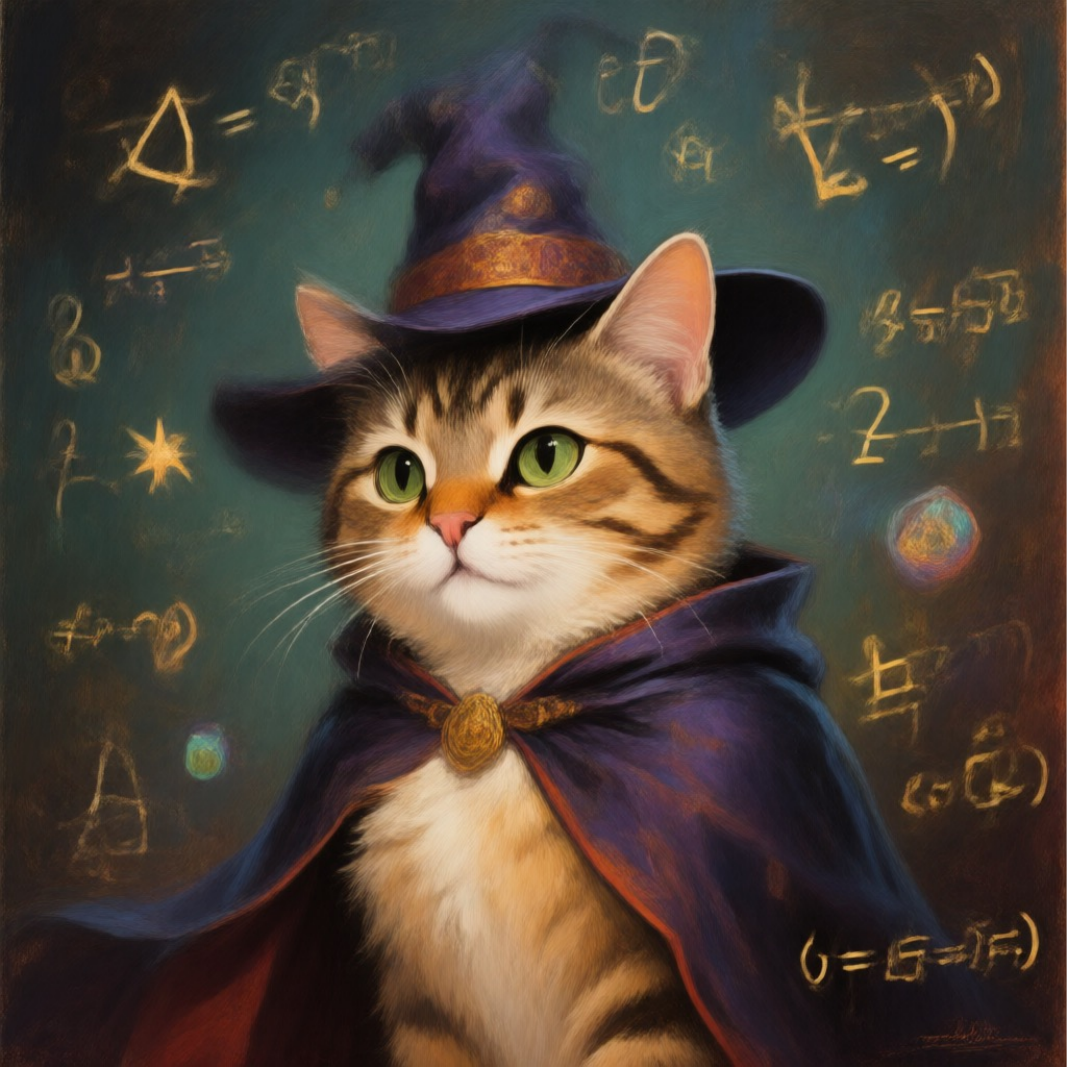}\hfill
        \includegraphics[width=0.49\textwidth]{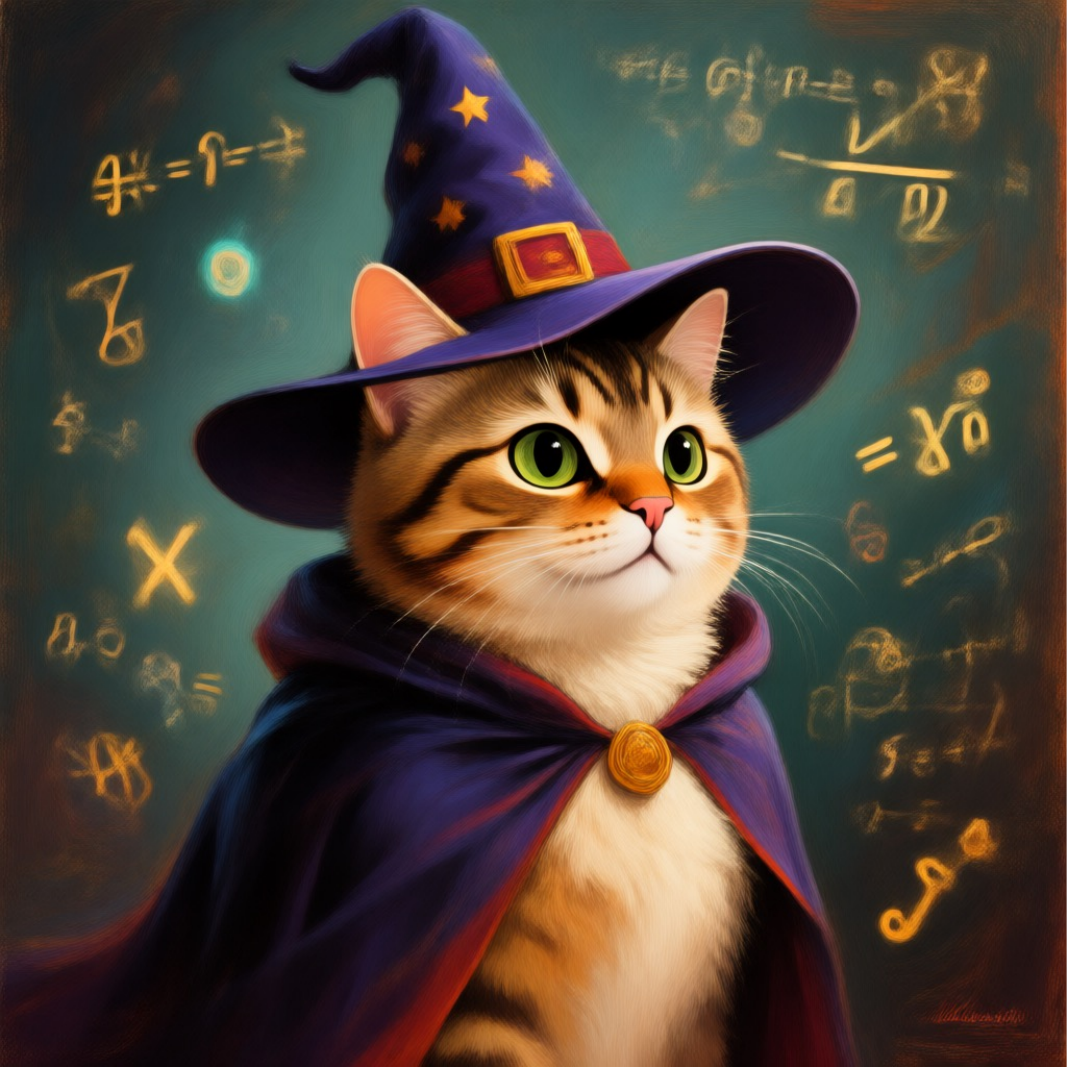}
    \end{minipage}\par\vspace{0.5mm}
    
    \begin{minipage}{0.64\textwidth}
        \includegraphics[width=\textwidth]{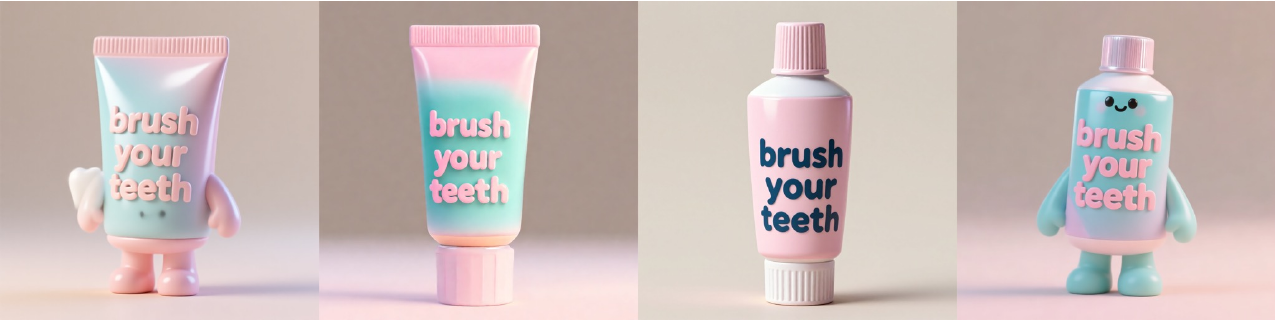}
    \end{minipage}%
    \begin{minipage}{0.02\textwidth}
        \centering
        $\vcenter{\hbox{\tikz[overlay] \draw [dashed, gray, thick] (0, 1.25) -- (0, -1.25);}}$
    \end{minipage}%
    \begin{minipage}{0.32\textwidth}
        \includegraphics[width=0.49\textwidth]{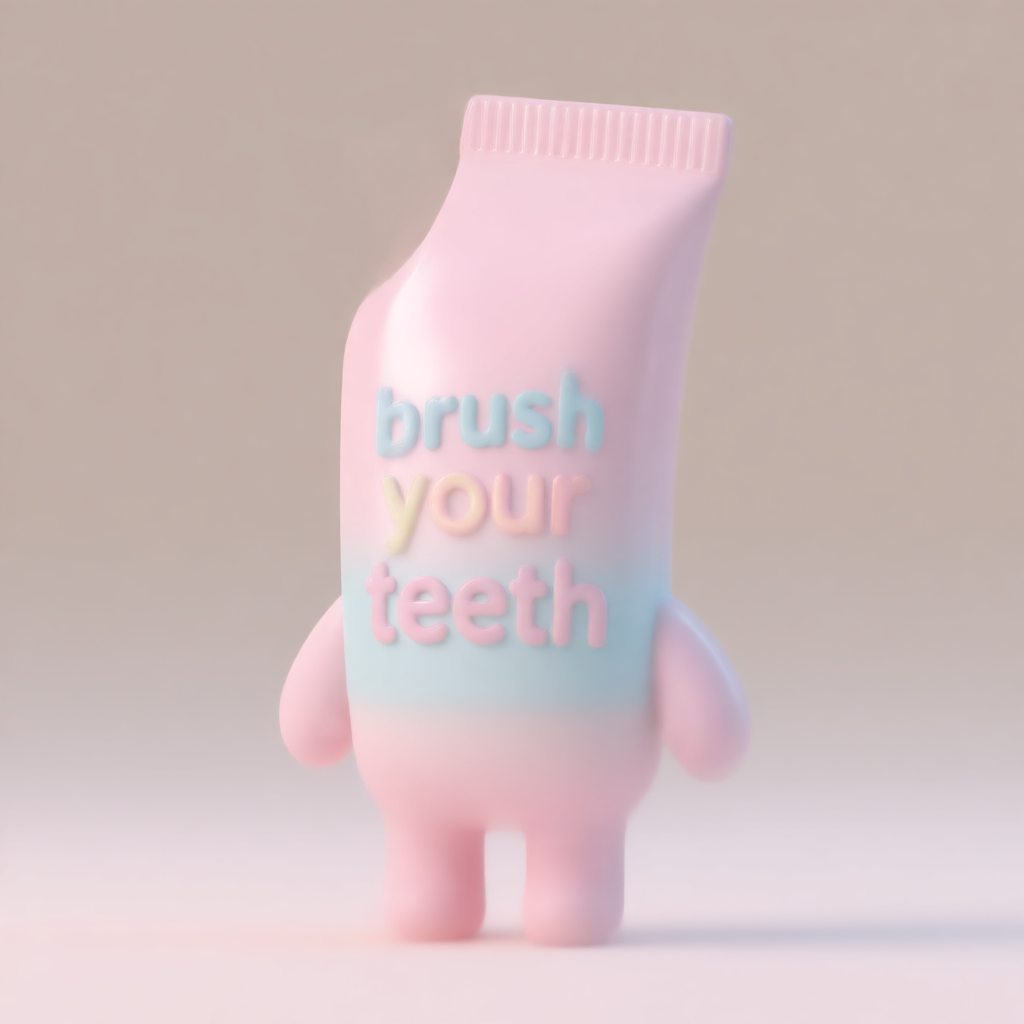}\hfill
        \includegraphics[width=0.49\textwidth]{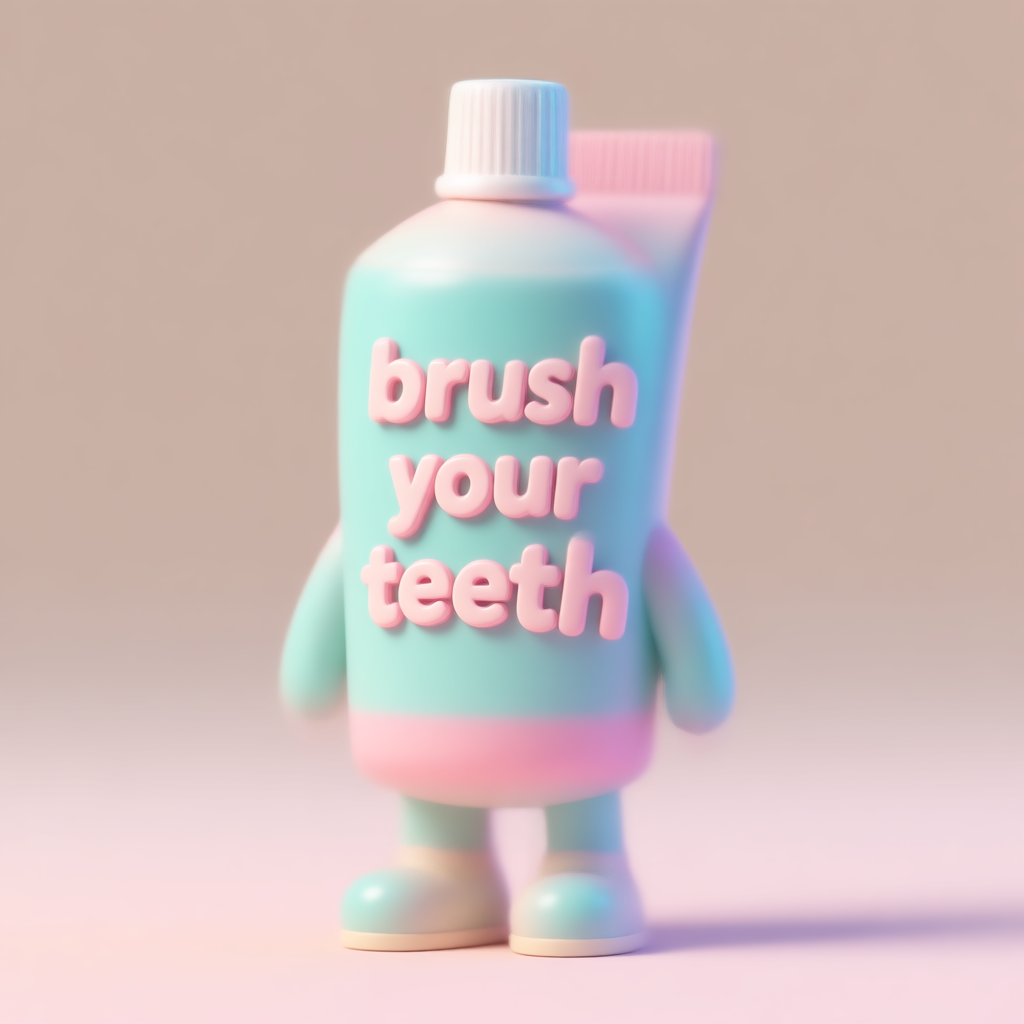}
    \end{minipage}\par\vspace{0.5mm}

    \caption{Qualitative comparison on DPG-Bench. IBFlow simultaneously improves \emph{structural fidelity}, \emph{semantic alignment}, \emph{chromatic fidelity}, and \emph{textural detail} over all 2-step baselines, while the rightmost two columns confirm the necessity of the dynamic target $\tau_{CA}^*$ and dynamic strength $\omega^*(t)$.}
    \vspace{-0.2cm}
    \label{fig:qualitative_dpg}
\end{figure*}

\paragraph{Qualitative comparison.}
Fig.~\ref{fig:qualitative_dpg} compares samples from Qwen-image-Lightning, ArcFlow, TwinFlow and IBFlow, with the two rightmost columns ablating $\tau_{CA}^*$ and $\omega^*(t)$. Static-CA baselines fail along four orthogonal dimensions: collapsed \emph{structural fidelity}, broken \emph{semantic alignment}, distorted \emph{chromatic fidelity}, and degraded \emph{textural detail}. IBFlow simultaneously improves all four, matching the multi-step teacher across structure, semantics, color and texture. Removing $\tau_{CA}^*$ predominantly damages structural fidelity and semantic alignment, while removing $\omega^*(t)$ predominantly damages chromatic fidelity and textural detail, providing a one-to-one visual confirmation of the complementary roles of our dynamic injection of target and strength.

\subsection{Ablation studies}
\label{subsec:ablation}

\begin{table}[t]
  \centering
  \footnotesize 
  \setlength{\tabcolsep}{6.5pt} 
  
  \begin{minipage}[t]{0.32\textwidth}
    \centering
    \caption{Component breakdown on Qwen-Image.}
    \label{tab:ablation_component}
    \begin{tabular}{l c c}
      \toprule
      \textbf{Method} & \textbf{Geneval} & \textbf{DPG} \\
      \midrule
      ArcFlow                  & 0.84 & 87.94 \\
      + Dyn $\tau_{CA}^*$      & 0.85 & 88.37 \\
      + Dyn $\omega^*(t)$      & 0.86 & 88.50 \\
      \rowcolor{gray!10} \textbf{IBFlow} & \textbf{0.86} & \textbf{88.67} \\
      \bottomrule
    \end{tabular}
  \end{minipage}\hfill
  \begin{minipage}[t]{0.30\textwidth}
    \centering
    \caption{Impact of divergence budget $\kappa$.}
    \label{tab:ablation_kappa}
    \begin{tabular}{l c c}
      \toprule
      \textbf{Budget ($\kappa$)} & \textbf{Geneval} & \textbf{DPG} \\
      \midrule
      $\kappa = 0.5$           & 0.85 & 88.16 \\
      $\kappa = 1.0$           & 0.86 & 88.32 \\
      \rowcolor{gray!10} \textbf{$\kappa = 1.5$} & \textbf{0.86} & \textbf{88.67} \\
      $\kappa = 3.0$           & 0.84 & 86.13 \\
      \bottomrule
    \end{tabular}
  \end{minipage}\hfill
  \begin{minipage}[t]{0.35\textwidth}
    \centering
    \caption{Comparison of $\omega$ scheduling.}
    \label{tab:ablation_omega}
    \begin{tabular}{l c c}
      \toprule
      \textbf{Schedule ($\omega$)} & \textbf{Geneval} & \textbf{DPG} \\
      \midrule
      Constant                 & 0.85 & 88.37 \\
      Linear Decay             & 0.85 & 87.82 \\
      Cosine Decay             & 0.86 & 88.16 \\
      \rowcolor{gray!10} \textbf{SNR-Analyt.} & \textbf{0.86} & \textbf{88.67} \\
      \bottomrule
    \end{tabular}
    
  \end{minipage}
  \vspace{-0.5cm}
\end{table}

\paragraph{The two dynamic schedules are complementary.}
Table~\ref{tab:ablation_component} isolates the effect of each dynamic component on Qwen-Image-20B under identical training budget and hyperparameters. Adding the dynamic injection target $\tau_{CA}^*$ alone improves ArcFlow from $0.84$ / $87.94$ to $0.85$ / $88.37$; adding the dynamic injection strength $\omega^*(t)$ alone improves it to $0.86$ / $88.50$; combining both yields the best $0.86$ / \textbf{88.67}. The two schedules thus address complementary failure modes: $\tau_{CA}^*$ prevents the student from bridging an excessive semantic gap during early stages and recovers high-frequency fidelity, whereas $\omega^*(t)$ removes the over-conditioning artifacts predicted by Theorem~\ref{the:omega}.

\paragraph{Choice of the divergence budget $\kappa$.}
Table~\ref{tab:ablation_kappa} sweeps the IB budget $\kappa$, which controls the maximum admissible KL divergence between the current state and the supervisor. A small $\kappa$ over-restricts the injection target and starves the student of useful supervision, while an excessively large $\kappa$ violates the local Fisher-information approximation in Lemma~\ref{lem:KL divergence constraint} and reintroduces distillation instability. The closed-form schedule attains its peak at $\kappa{=}1.5$, where IBFlow recovers the full $0.86$ / \textbf{88.67}.


\section{Conclusion}
\label{sec:conclusion}

In this paper, we proposed IB-Flow, a novel Information Bottleneck-guided dynamic CFG distillation framework that significantly elevates the performance boundary of few-step text-to-image generation. Existing frameworks rely on blind injection paradigm that enforces static guidance strength while indiscriminately sampling the supervisor target. To overcome these bottlenecks, we introduced a dual-track adaptive mechanism: an instance-aware target selection strategy that guarantees conservative macroscopic structural anchoring during early high-entropy stages, and an entropy-aware guidance strength schedule that smoothly decays to the natural manifold to eradicate CFG over-conditioning artifacts. Extensive empirical evaluations across multiple large-scale backbones demonstrate that IB-Flow achieves state-of-the-art generative fidelity under extremely 2-step configurations, effectively restoring structural and chromatic fidelity.

\clearpage
\newpage

\clearpage
{
\small
\bibliographystyle{cite}
\bibliography{main}
}


\appendix

\clearpage
\newpage

\appendix

\section{Proofs of the Information-Bottleneck Formulation}
\label{sec:appendix_proof}

This appendix provides the full proofs of Lemma~\ref{lem:KL divergence constraint}, Lemma~\ref{lem:injection span}, Theorem~\ref{the:tau}, and Theorem~\ref{the:omega} stated in the main text. The presentation strictly follows the order used in Section~\ref{sec:methodology}: each subsection corresponds to exactly one statement, and the labels match those referenced from the main paper.

\subsection{Proof of Lemma~\ref{lem:KL divergence constraint} (Local KL Bound via Fisher Information)}
\label{sec:appendix_lemma1}

We prove that, under the Flow Matching framework, the local KL divergence between marginals at $t$ and $\tau = t + \Delta t$ admits the approximation
\begin{equation*}
    D_{KL}(p_t \| p_\tau) \approx \tfrac{1}{2}\,(\tau-t)^2 \, \|\Delta v_t(x_t)\|_2^2.
\end{equation*}

For continuous normalizing flows, the marginal density $p_t(x_t)$ satisfies the continuity equation
\begin{equation}
    \partial_t p_t + \nabla \cdot (p_t v_t) = 0.
\end{equation}
For Gaussian probability paths used in FM (e.g., OT-Flow or VP-Flow), the predicted vector field $v_t$ is intrinsically coupled with the Stein score $\nabla_{x_t}\log p_t$ via a time-dependent scalar weighting $\sigma_t$:
\begin{equation}
    v_t(x_t) \propto \mathbb{E}[X_1\mid X_t = x_t] - x_t \propto \sigma_t^2\,\nabla_{x_t}\log p_t(x_t).
    \label{eq:appendix_score_relation}
\end{equation}

The CFG residual is $\Delta v_t = v_t(x_t,c) - v_t(x_t,\emptyset)$. Substituting \eqref{eq:appendix_score_relation},
\begin{align}
    \Delta v_t(x_t)
    &\propto \nabla_{x_t}\log p_t(x_t\mid c) - \nabla_{x_t}\log p_t(x_t) \nonumber \\
    &= \nabla_{x_t}\log\frac{p_t(x_t,c)}{p_t(x_t)\,p(c)}
    = \nabla_{x_t}\,\mathrm{PMI}_t(x_t;c),
\end{align}
so $\|\Delta v_t\|_2$ is proportional to the spatial gradient of the pointwise mutual information injected by the condition $c$.

By a second-order Taylor expansion in $\Delta t$, the KL divergence between two infinitesimally close distributions on the same parametric path satisfies
\begin{equation}
    D_{KL}(p_t \| p_{t+\Delta t}) \approx \tfrac{1}{2}(\Delta t)^2 \, \mathcal{I}_F(t),
    \quad
    \mathcal{I}_F(t) = \mathbb{E}_{p_t}\!\left[\,\Big\|\partial_t \log p_t\Big\|^2\right].
\end{equation}
For the CA-driven transition, the dominant contribution to the temporal score variation is the conditional shift induced by $\Delta v_t$. Translating this temporal variation into the spatial PMI gradient gives
\begin{equation}
    D_{KL}^{CA}(p_t\|p_\tau) \approx \tfrac{1}{2}(\tau-t)^2\,\big\|\nabla_{x_t}\mathrm{PMI}_t(x_t;c)\big\|_2^2 \propto \tfrac{1}{2}(\tau-t)^2\,\|\Delta v_t(x_t)\|_2^2,
\end{equation}
which establishes Lemma~\ref{lem:KL divergence constraint}. \qed

\subsection{Proof of Lemma~\ref{lem:injection span} (Admissible Injection Span)}
\label{sec:appendix_lemma2}

Given the IB budget $D_{KL}(p_t\|p_\tau)\le \delta$, we plug the approximation from Lemma~\ref{lem:KL divergence constraint} (proved in Appendix~\ref{sec:appendix_lemma1}) into the constraint:
\begin{equation}
    \tfrac{1}{2}(\tau-t)^2\,\|\Delta v_t(x_t)\|_2^2 \;\le\; \delta.
\end{equation}
Since $(\tau-t)\ge 0$ and $\|\Delta v_t(x_t)\|_2 \ge 0$, taking square roots and rearranging yields
\begin{equation}
    \tau - t \;\le\; \frac{\sqrt{2\delta}}{\|\Delta v_t(x_t)\|_2},
\end{equation}
which is the admissible injection span stated in Lemma~\ref{lem:injection span}. The bound is instance-aware: a large CFG residual (sharp local conditional gradient) tightens the admissible reach, while a small residual permits looking further ahead. \qed

\subsection{Proof of Theorem~\ref{the:tau} (Closed-Form Optimal Injection Target)}
\label{sec:appendix_theorem1}

We prove the closed-form solution
\begin{equation*}
    \tau_{CA}^*(x_t,t) = \min\!\left(1,\; t + \frac{\kappa}{\|v_t^{c}(x_t) - v_t^{u}(x_t)\|_2 + \epsilon}\right),\quad \kappa = \sqrt{2\delta}.
\end{equation*}

Under flow matching, the forward process $\{X_\tau\}$ is a monotone information channel from $X_t$ towards $X_1$, so
\begin{equation}
    \tau_1 \le \tau_2 \;\Rightarrow\; I(X_{\tau_1};X_1\mid X_t) \le I(X_{\tau_2};X_1\mid X_t).
\end{equation}
The IB objective \eqref{eq:ib_optimization} is therefore non-decreasing in $\tau$, and its optimum is attained at the largest $\tau$ permitted by the constraint.

Combining this monotonicity with the admissible span of Lemma~\ref{lem:injection span}, the unconstrained-by-terminal optimum is obtained:
\begin{equation}
    \tau^\diamond = t + \frac{\sqrt{2\delta}}{\|\Delta v_t(x_t)\|_2}.
\end{equation}

The supervisor timestep is bounded above by the terminal time $\tau \le 1$, since $X_1$ is the clean data state. Clipping $\tau^\diamond$ from above and adding a small $\epsilon>0$ to the denominator to prevent division by zero, we obtain
\begin{equation}
    \tau_{CA}^*(x_t,t) = \min\!\left(1,\; t + \frac{\kappa}{\|v_t^{c}(x_t) - v_t^{u}(x_t)\|_2 + \epsilon}\right),
\end{equation}
with $\kappa=\sqrt{2\delta}$, which proves Theorem~\ref{the:tau}. \qed

\subsection{Proof of Theorem~\ref{the:omega} (Closed-Form Optimal Injection Strength)}
\label{sec:appendix_theorem2}

We prove the SNR-driven schedule
\begin{equation*}
    \omega^*(t) = 1 + (\omega_{\max}-1)\cdot\frac{1}{1+\gamma\,\mathrm{SNR}(t)},
    \quad
    \mathrm{SNR}(t)=\frac{t^2}{(1-t)^2}.
\end{equation*}

The IB condition-injection objective \eqref{eq:ib_condition_optimization} contains the generative-fidelity term $I(X_t^{\omega};X_1\mid C)$. By the chain rule of conditional mutual information,
\begin{equation}
    I(X_t^{\omega};X_1\mid C) \;=\; H(X_1\mid C) \;-\; H(X_1\mid X_t^{\omega}, C).
    \label{eq:appendix_mi_chain}
\end{equation}
The first term $H(X_1\mid C)$ is an intrinsic property of the data--condition joint distribution and is therefore independent of the guidance schedule $\omega$. Consequently, the only $\omega$-dependent component of the fidelity term is the negative residual entropy $-H(X_1\mid X_t^{\omega}, C)$, which means the achievable fidelity gain obtained by raising $\omega$ is upper-bounded by the residual uncertainty about $X_1$ left after observing $X_t^{\omega}$ and $C$:
\begin{equation}
    \Delta I(X_t^{\omega};X_1\mid C) \;\le\; H(X_1\mid X_t^{\omega}, C) \;\le\; H(X_1\mid X_t),
    \label{eq:appendix_mi_bound}
\end{equation}
where the second inequality follows because conditioning on additional information (the condition $C$ on top of $X_t$) cannot increase entropy. Equation~\eqref{eq:appendix_mi_bound} shows that the marginal fidelity gain from $\omega$ is governed by $H(X_1\mid X_t)$, the residual uncertainty about the clean data given the current state.

By the I-MMSE relation~\cite{guo2005mutual}, when the forward trajectory is locally modelled as an additive Gaussian noise channel, the posterior variance of estimating $X_1$ from $X_t$ scales inversely with the instantaneous signal-to-noise ratio. Up to a logarithmic factor,
\begin{equation}
    H(X_1\mid X_t) \;\propto\; \frac{1}{1+\mathrm{SNR}(t)}.
    \label{eq:appendix_snr_uncertainty}
\end{equation}

Under the Rectified Flow path $X_t = tX_1 + (1-t)X_0$ with $X_0\sim\mathcal{N}(0,I)$, the signal amplitude is exactly $t$ and the noise amplitude is $1-t$, giving
\begin{equation}
    \mathrm{SNR}(t) = \frac{\|tX_1\|_2^2}{\|(1-t)X_0\|_2^2} = \frac{t^2}{(1-t)^2}.
    \label{eq:appendix_snr_calculation}
\end{equation}
To minimize the IB cost without inducing manifold distortion via over-conditioning, the supplementary CFG drive $\Delta\omega(t) = \omega(t)-1$ should scale with the residual uncertainty in \eqref{eq:appendix_snr_uncertainty}. Anchoring the maximum drive at the high-entropy initial stage with $\omega(0)=\omega_{\max}$ and substituting \eqref{eq:appendix_snr_calculation}, we obtain
\begin{equation}
    \omega^*(t) = 1 + (\omega_{\max}-1)\cdot\frac{1}{1+\gamma\,\mathrm{SNR}(t)},
\end{equation}
with $\gamma>0$ controlling the SNR sensitivity. As $t\to 1$, $\mathrm{SNR}(t)\to\infty$ and $\omega^*(t)\to 1$, so the guidance smoothly retracts to the unconditional manifold and CFG over-conditioning artifacts are eliminated at the source. This proves Theorem~\ref{the:omega}. \qed

\section{Compute Resources}
\label{app:compute}
Each main training in the manuscript is conducted on 32 NVIDIA A100 GPUs with 80GB memory and takes approximately 1.5 days. We report the compute configuration for the experiments in the paper, while additional exploratory runs may require extra compute beyond the final reported settings.

\section{Limitations}
\label{sec:appendix_limitations}

Despite the superior efficiency and fidelity achieved by IB-Flow, several limitations remain to be addressed in future work:

\paragraph{Upper Bound of Teacher Distribution.} 
As a distillation-based framework, IB-Flow primarily focuses on compressing the iterative sampling trajectory of a pre-trained teacher model. While our information-theoretic dynamic schedules significantly reduce over-conditioning artifacts and bridge the semantic gap, the fundamental generative capacity remains upper-bounded by the teacher's distribution and the training corpus. Since we maintain the same dataset as the teacher to ensure a fair comparison, the model does not exhibit a radical shift in the diversity or the underlying data distribution beyond what the teacher has already mastered.

\paragraph{Specific to Flow-based Trajectories.} 
Our current mathematical derivation for the dynamic injection target $\tau_{CA}^*$ and strength $\omega^*$ is optimized for Flow Matching and Rectified Flow frameworks. While the Information Bottleneck principle is universal, extending these closed-form solutions to traditional SDE-based diffusion models might require additional approximations for the non-linear variance schedules.

\section{Broader Impacts}
\label{sec:appendix_impacts}

The development of IB-Flow carries both positive societal benefits and potential ethical risks:

\paragraph{Positive Impacts.} 
By reducing the computational cost of high-quality image generation by $25\times$ to $50\times$, our framework democratizes access to state-of-the-art AI. This significantly lowers the hardware barrier for creative individuals and researchers, reducing the carbon footprint associated with large-scale inference and enabling real-time interactive applications on edge devices.

\paragraph{Negative Impacts.} 
As with any advanced generative technology, the enhanced fidelity and extreme speed of IB-Flow could be misused for the rapid creation of deceptive synthetic media or "Deepfakes." The ability to generate near-photorealistic images in only two steps complicates the task of real-time detection and could potentially be exploited for disinformation campaigns. We encourage the community to integrate robust digital watermarking and provenance standards (e.g., C2PA) to mitigate these risks.



\end{document}